%% file: 2022-STMAML-AAAI.tex
\def\year{2022}\relax
\documentclass[letterpaper]{article} %
\usepackage{aaai22}  %
\usepackage{times}  %
\usepackage{helvet}  %
\usepackage{courier}  %
\usepackage[hyphens]{url}  %
\usepackage{graphicx} %
\usepackage{natbib}  %
\usepackage{caption} %
\DeclareCaptionStyle{ruled}{labelfont=normalfont,labelsep=colon,strut=off} %
\usepackage{algorithm}
\usepackage{algorithmic}
\usepackage{graphicx}
\graphicspath{{figures/}{../figures/}{figs/}{../figs/}}
\input{preambular}

\usepackage{newfloat}
\usepackage{listings}
\floatstyle{ruled}
\newfloat{listing}{tb}{lst}{}
\floatname{listing}{Listing}
\title{\method : A Stochastic-Task based Method for Task-Heterogeneous Meta-Learning}
\author{
    Zhe Wang, Jake Grigsby, Arshdeep Sekhon, Yanjun Qi
}
\affiliations{

    University of Virginia\\
    \{zw6sg, jcg6dn, as5cu, yq2h\}@virginia.edu

}

\begin{document}

\maketitle

\input{files/00abs.tex}

\input{files/01intro.tex}

\section{Methods}

\input{files/03-motivation-figure}

\input{files/03pre.tex}

\input{files/04-0task}

\input{files/04-1model}

\input{files/04-2loss}

\input{files/02related.tex}

\input{files/05-2tables.tex}

\input{files/05-1exp.tex}

\input{files/08conclude.tex}

\FloatBarrier

\bibliography{files/meta-ref}

\appendix

\onecolumn
\section{Appendix}
\input{files/99-related-table.tex}
\input{files/99-theory.tex}

\input{files/99-moreExp.tex}

\input{files/99-extra-figures.tex}
\input{files/100-equations.tex}
\end{document}

%% file: preambular.tex
\usepackage{fancyhdr}
\usepackage{color}
\usepackage{url}
\usepackage{multirow}
\usepackage{graphicx} %
\usepackage{times}
\usepackage{latexsym}
\usepackage{multicol}
\usepackage{multirow}

\usepackage{placeins}

\usepackage{tabularx} %
\usepackage{amssymb}%
\usepackage{amsmath}
\usepackage{mathtools}
\usepackage{pifont}%

\usepackage{ragged2e}

\usepackage{hyperref}
\usepackage{url}
\usepackage{dirtytalk}

\newcommand{\sref}[1]{Section~(\ref{#1})} 
\newcommand{\eref}[1]{Eq.~(\ref{#1})}

\newcommand{\cref}[1]{Condition~(\ref{#1})}

\include{macros}
 
\def\x{{\mathbf x}} 
 
\def\y{{\mathbf y}} 
\def\X{{\mathbf X}}

\def\s{{\mathbf s}} 
\def\h{{\mathbf h}}

\def\x{{\mathbf x}} 
 
\def\y{{\mathbf y}} 
\def\s{{\mathbf s}} 
\def\h{{\mathbf h}}

\usepackage{amsfonts,amssymb,bm}
\usepackage{pifont} %
\usepackage{graphicx,subfigure,epsfig,fancybox} %
\usepackage{float}
\usepackage{color} %
\usepackage{multirow}
\usepackage{natbib}

\newcommand{\revise}[1]{\textcolor{blue}{}}
\newcommand{\revised}[1]{\textcolor{blue}{}}
\renewcommand{\vec}[1]{\boldsymbol{#1}}

\newcommand{\vw}[0]{\vec{w}}

\newcommand{\vf}[0]{\vec{f}}
\newcommand{\vg}[0]{\vec{g}}
\newcommand{\vx}[0]{\vec{x}}
\newcommand{\vz}[0]{\vec{z}}

\newcommand{\vph}[0]{\vec{\phi}}

\newcommand{\vr}[0]{\vec{r}}
\newcommand{\vbeta}[0]{\vec{\beta}}

\newtheorem{lemma}{Lemma}
\newtheorem{proof}{Proof}
\newcommand{\task}{\mathcal{T}}
\newcommand{\likeli}{\mathcal{L} ({\task})}

\newcommand{\modelnop}{\vf}
\newcommand{\model}{\vf_{\goal_{\task}}}

\newcommand{\loss}{\mathcal{L}oss}

\newcommand{\goal}{\vec{\theta}}
\newcommand{\datatr}{\vec{D}^{tr}_{\task}}
\newcommand{\datats}{\vec{D}^{te}_{\task}}
\newcommand{\biggoal}{\vec{\Theta}_{\task}}
\newcommand{\taskZ}{\vec{Z}_{\task}}
\newcommand{\datayitr}{\vec{Y}_{\task}^{tr}}
\newcommand{\dataxitr}{\vec{X}_{\task}^{tr}}
\newcommand{\datayite}{\vec{Y}_{\task}^{te}}
\newcommand{\dataxite}{\vec{X}_{\task}^{te}}
\newcommand{\goali}{\goal_{\task}}

\usepackage{adjustbox}
\usepackage{array}
\usepackage{booktabs}

\newcommand{\method}[0]{\texttt{ST-MAML} }

\usepackage{enumitem}
\setlist[itemize]{leftmargin=*}

%% file: files/00abs.tex
\begin{abstract}
\label{abs}

Optimization-based meta-learning typically assumes tasks are sampled from a single distribution -- an assumption oversimplifies and limits the diversity of tasks that meta-learning can model. 
Handling tasks from multiple different distributions is challenging for meta-learning due to a so-called task ambiguity issue. This paper proposes a novel method, \texttt{ST-MAML}, that empowers model-agnostic meta-learning (\texttt{MAML}) to learn from multiple task distributions. \method encodes tasks using a stochastic neural network module, that summarizes every task with a stochastic representation. The proposed Stochastic Task (\texttt{ST}) strategy allows a meta-model to get tailored for the current task and enables us to learn a distribution of solutions for an ambiguous task. \method also propagates the task representation to revise the encoding of input variables. Empirically, we demonstrate that \method matches or outperforms the state-of-the-art on two few-shot image classification tasks, one curve regression benchmark, one image completion problem, and a real-world temperature prediction application. 
To the best of authors' knowledge, this is the first time optimization-based meta-learning method being applied on a large-scale real-world task.%

\end{abstract}

%% file: files/01intro.tex
\section{Introduction}
\label{intro}

Meta-learning aims to train a model on multiple machine learning tasks to adapt to a new task with only a few training samples. Optimization-based meta-learning  like model-agnostic meta-learning (MAML) facilitate such a goal by involving the optimization process. For example, MAML trains a global initialization of model parameters that are close to the optimal parameter values of every task ~\cite{finn2017model}.  Recent methods expand MAML's "global initialization" to a notion of "globally shared knowledge," including not only  initialization~\cite{finn2017model, li2017meta, rajeswaran2019meta} but also update rules~\cite{andrychowicz2016learning,  Ravi2017OptimizationAA}. The globally shared knowledge
are explicitly trained and allow these methods to produce good generalization performance on new tasks with a small number of training samples.

Most optimization-based meta-learning algorithms assume all tasks are identically and independently sampled from a single  distribution~\cite{andrychowicz2016learning, finn2017model, li2017meta, Ravi2017OptimizationAA, rusu2018meta}. This setup is  known as task homogeneity. We name meta-learning's target task distribution as ``meta-distribution".  Real-world tasks, however, may come from multiple meta-distributions. For instance, autonomous driving agents need to be able to handle multiple learning environments, including those under different lighting, various weather situations, and a diverse set of road shapes. This more challenging setup, we call task heterogeneity, posts technical challenges to strategies like MAML ~\cite{vuorio2019multimodal}.

For task heterogeneity setup, a naive and widely accepted meta-learning solution first learns a globally shared initialization across all meta-distributions and then tailors the model parameter to the current task~\cite{vuorio2019multimodal, yao2020automated, yao2019hierarchically, lee2018gradient, oreshkin2018tadam}. The tailoring step needs to rely on the task-specific information or, ideally, the identity information of the task. 
It, therefore, requires the meta-learner to infer the potential identity of a new task from a limited number of annotated samples~\cite{finn2018probabilistic}. This requirement raises severe uncertainty issues -- a challenge known as "task ambiguity." Figure \ref{fig:abs_demostration} provides a concrete example of "task ambiguity" that attributes to not only the limited annotated data but also from the multiple distributions that a task may come.  Surprisingly, recent optimization-based meta-learning literature pay little attention to the task ambiguity challenge ~\cite{vuorio2019multimodal, yao2020automated, yao2019hierarchically, lee2018gradient}.

\begin{figure*}[th]
  \centering
  \includegraphics[width=0.8\textwidth]{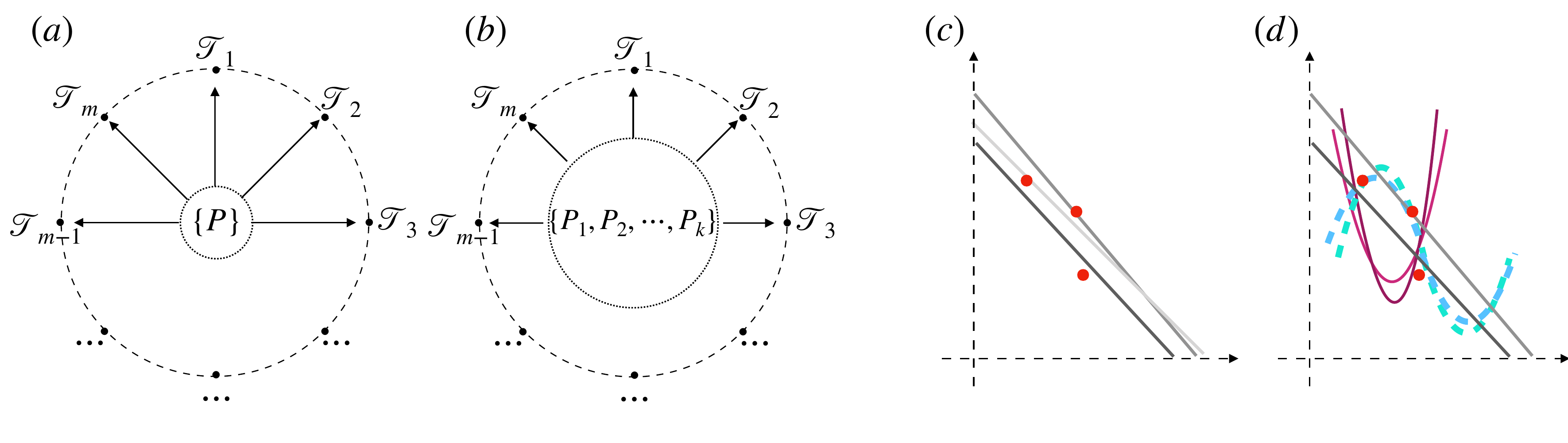}
  \caption{
    	\small
	Two critical challenges in meta-learning. 
	(a, b): The figures show the difference between task homogeneity and task heterogeneity in meta-learning. The solid line with arrow represents the uniformly random sampling from meta distributions (inner circle).
	(c, d): The figures demonstrate the task ambiguity in meta-learning. In heterogeneous setup, the task ambiguity is more critical due to the distributional uncertainty. The red dots represent the available training data, the dashed and solid curves are potential explanations of the data (better read in color).
	    \label{fig:abs_demostration}
	}
\end{figure*}

This paper proposes a novel meta-learning method \method for task heterogeneity challenge and centers our design on solving the task ambiguity issue. Our approach extends MAML by modeling tasks as a stochastic variable that we name as stochastic task. Stochastic task  allows us to learn  a distribution of models to capture the uncertainty of an ambiguous new task. We use variational inference as solver and the whole learning process does not require knowing the cardinality of meta-distributions. We apply the \method on multiple applications, including image completion, few-shots image classification, and temporal forecasting meta-learning problems. To the best of authors' knowledge, this is the first time optimization-based meta-learning being applied on a large-scale real-life task. Our empirical results demonstrate that \method outperforms the MAML baselines with $40\%$ on that task.

%% file: files/03-motivation-figure.tex
\input{files/102-figandalg.tex}

%% file: files/102-figandalg.tex
\begin{figure*}[ttt!]

\begin{tabular}{cc}
\begin{minipage}{.5\textwidth}
    
    \centering
    \includegraphics[width=\textwidth]{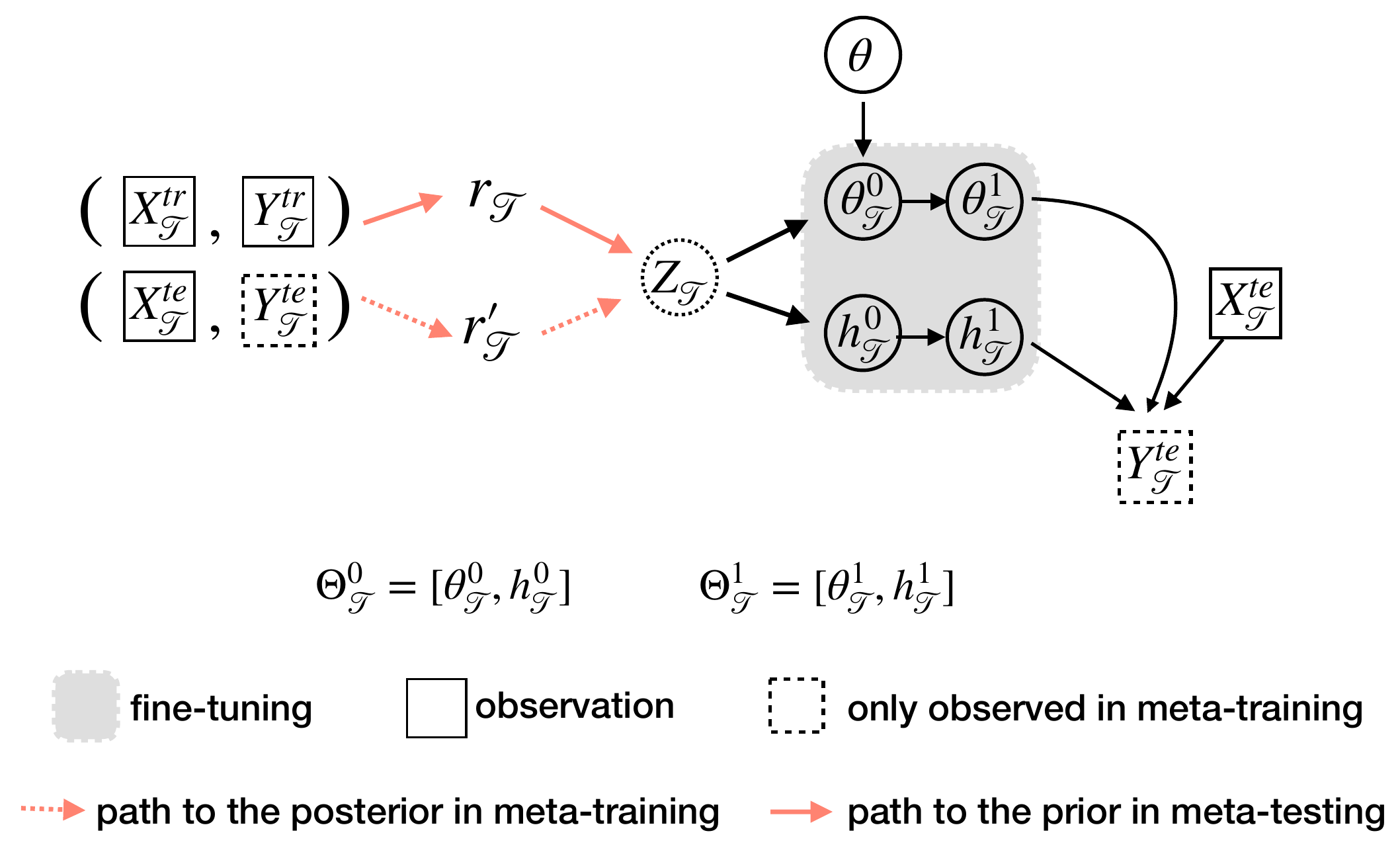}
    \captionof{figure}{
        Probabilistic model overview of \method. 
    } \label{fig:pgraph}

\end{minipage} &

\hfill

\scalebox{0.72}{
    \begin{minipage}{.66\textwidth}
    \begin{algorithm}[H]
    \captionof{algorithm}{\textsc{\method Meta-Training Procedure.}} \label{alg:train}
        \begin{algorithmic}[1]
        \STATE {\bfseries Input:} Meta-distributions $\{P_1(\mathcal{T}), \cdots, P_k(\mathcal{T})\} $, Hyper-parameters $\gamma_1$ and $\gamma_2$.
    	\STATE Randomly initialize model parameter $\goal$, stochastic task module parameters $\vec{\phi}$, tailoring module parameters $\vec{w}$, input encoding parameters $\vbeta$.
    		\WHILE{not DONE}
        		\STATE Sample batches of $m$ tasks $\{\mathcal{T}\}$ from meta-distributions.
        		\FOR{every task $\mathcal{T}$}
        		    \STATE Infer the posterior distribution of stochastic task variable $q(\taskZ|\task)$ and sample $\vz_{\task}\sim q(\taskZ|\task)$. [eq.\eqref{eq:Rtask} and eq.\eqref{eq:postz}]
		    \STATE Tailor $\goal$  with sample $\vz_{\task}$ to get task-specific initialization $\goali^0$. [eq.\eqref{eq:trans1}]
		    \STATE Revise the encoding of input variable by augmenting the raw input. [eq.\eqref{eq:trans2}]
        		    \STATE Evaluate the inner loss $\mathcal{L}_{in}(\task)$ on training set $\datatr$.  [eq.\eqref{eq:inner_loss}]
        		    \STATE Compute adapted parameter and augmented feature with gradient descent [eq.\eqref{eq:inner_update}]: \\ $\goali^1 = \goali^0 - \gamma_1 \nabla_{\goali^0}\mathcal{L}_{in}(\task)$,\ 
        		    $ \h_{\task}^{1} = \h_{\task}^{0} - \gamma_1\nabla_{\h_{\task}^0}\mathcal{L}_{in}(\task)$. 
        		\ENDFOR
        		\STATE Update $\goal, \vec{\phi}, \vec{w}, \vec{\beta}$ with $\gamma_2\dfrac{1}{m} \nabla_{[\goal, \vec{\phi}, \vec{w}, \vec{\beta}]}\sum_\task \mathcal{L}_{ELBO}(\task)$. [eq.\eqref{eq:LIB}]
        		
    		\ENDWHILE  
    \end{algorithmic}

    \end{algorithm}
    
    \end{minipage}
}

\end{tabular}

\end{figure*}

%% file: files/03pre.tex
\subsection{Preliminaries on Meta Learning}

We describe a supervised learning task in meta-learning as 
\begin{align}
\task & = \{ \loss(), \model, \datatr, \datats \} \nonumber \\
& = \{ \loss(), \model, [\dataxitr, \datayitr], [\dataxite, \datayite]\},
\end{align}
Here $\loss()$, which takes as input model $\model$ and dataset,  describes the loss function that measures the quality of learner $\model$, whose parameter weight is $\goali$. Every task includes an annotated training set $\datatr = [\dataxitr, \datayitr]$ and a test set $\datats =[\dataxite, \datayite]$.  During meta-training, the test set $\datats$ is fully observed, but during meta-testing only its input $\dataxite$ is available. $\datatr$ and $\datats$ are sampled from $ \mathcal{\vec{X}}\times \mathcal{\vec{Y}}$, $\mathcal{\vec{X}}$ describes the input space and $\mathcal{\vec{Y}}$ is the output space.

The goal of meta learning is that on every task, the learner machine $\model$ needs to perform well on $\datats$ after fine-tuning on this task's training set $\datatr$. MAML \cite{finn2017model} achieves such a goal by learning a globally shared weight initialization  $\goal^*$ that is close to the optimal weight parameter of every task. We can write its training objective for getting the best initialization $\goal^*$ as:
\begin{align}
& \min_{\goal} \mathop{\mathbf{E}} \limits_{\task \sim P(\task) } [ \loss ( \vf_{\goali^{1}}, \datats) ] , \nonumber  \\
& \text{where\quad }
 \goali^{1} = \goali^{0} - \alpha \nabla_{\goal} [ \loss  ( \vf_{\goali^{0}},  \datatr)],\nonumber \\
& \text{and\quad }  \goali^0 = \goal.
\label{eq:maml_obj}
\end{align}

MAML samples a set of tasks $\{\task \}$ from the meta distribution $P(\task)$ and initialize each task's weight $\goali^{0}$ from the global knowledge $\goal$ (to be learnt): i.e., setting $\goali^{0} = \goal$. On each task, the learner performs gradient descent on its training set $\datatr$ to reach task-specific fine-tuned parameters $\goali^1$. The test set $\datats$ of task $\task$ is used for evaluating the current parameter $\goali^1$, and the evaluation will be used as the objective to optimize for learning the best global knowledge $\goal$.

The above objective (in \eref{eq:maml_obj}) can be equivalently framed as maximizing the likelihood $\likeli$:
\begin{align}
&\max_{\goal} \mathop{\mathbf{E}} \limits_{\task \sim P(\task) } [ \likeli ]   =   \prod\limits_{\task\sim P(\task)}p(\datayite|\dataxite, \datatr, \goal) \label{eq:likeli} \\ &= \prod\limits_{\task\sim P(\task)}\sum_{\goali^1}p(\datayite | \dataxite, \goali^1)p(\goali^1 | \datatr, \goal),
\end{align}
where $p(\goali^1 | \datatr, \goal)$ is a Dirac distribution derived by minimizing the negative log-likelihood(NLL) on $\datatr$ with gradient descent.

\subsection{Previous Heterogeneous Meta Learning}

Task-homogeneous meta-learning assumes that there exists one meta-distribution $P(\task)$ and all tasks are identically and independently (i.i.d.) sampled from $P(\task)$. Differently, in a task-heterogeneous setup, there exist multiple meta-distributions $\task \sim \{P_1(\task), P_2(\task), \cdots, P_k(\task)\}$. Figure \ref{fig:abs_demostration} (a,b) compare two described meta-learning setups. 

We can naively use MAML and assign all tasks with the same global initialization (though they come from different distributions). 
Figure \ref{fig:abs_demostration}(c, d) show that the "task ambiguity" issue is more critical in task-heterogeneous setup and will hinder the generalization from MAML initialization since multiple very different task distributions exist.

A handful of previous works learn a customized initialization that was tailored from global initialization, in order to tackle the task heterogeneity challenge. MMAML~\cite{vuorio2019multimodal} learns a deterministic task embedding with an RNN module. %
HSML~\cite{yao2019hierarchically} manually designs a task clustering algorithm to assign tasks to different clusters, then customizes the global initialization to each cluster.  ARML~\cite{yao2020automated} models global knowledge and task-specific knowledge as graphs; the interaction between tasks is modeled by message passing.

Surprisingly, none of the recent works consider the task ambiguity issue. Most frameworks are still based on the assumption that  only one distribution exists to explain a task's observed training set (e.g., a new task should be assigned to only one cluster in HSML).  The potential identities of a task can be highly uncertain under the limited annotated data scenario. Figure \ref{fig:abs_demostration}(d) shows that the explanation of the observation can be various in task-heterogeneous setup and  we should not expect to obtain a unique predictor.

%% file: files/04-0task.tex
\subsection{Stochastic Variable $Z_{\task}$ to Encode Task}

When facing the task-heterogeneous setup, we hypothesize that a meta-learner that can encode potential tasks' patterns will alleviate the task ambiguity issue (to some degrees). These patterns could describe valuable information about tasks like the more possible shapes of curves for a regression meta-application. Moreover, we propose to enable task encoding with uncertainty estimates. This is because learning a task representation from its limited annotated data is challenging and such uncertainty measures can help inform the downstream meta-adaptation to new tasks (see Figure~\ref{fig:abs_demostration}(d)).

This hypothesis motivates us to describe a task $\task$ with a stochastic variable $\taskZ$ and model its distribution to condition on observations. With adding this latent variable, we can rewrite the per task likelihood $\likeli$ in \eref{eq:likeli} as:  
\begin{align}
\likeli & = \sum_{\taskZ} p(\datayite|\dataxite, \datatr, \taskZ, \goal) p(\taskZ |\datatr).
\label{eq:likeliz}
\end{align}
We assume in the second term from above, $\taskZ$ only conditions on  $\datatr$. Figure~\ref{fig:pgraph} shows our design.

In later \sref{sec:update}, we show that due to the intractable likelihood as defined above, we choose to maximize its evidence lower bound (a.k.a ELBO) instead. Optimizing this variational objective requires the prior $p(\taskZ|\datatr)$ and the posterior $ q(\taskZ| \task)$. We model the prior $p(\taskZ|\datatr)$ as a Gaussian distribution, whose mean and variance are outputs from a two-layer multi-layer perceptron (MLP) module with input vector $\vr_\task$:
\begin{equation}
     p(\taskZ|\datatr) = \mathcal{N}(\vec{\mu}(\vr_{\task}), \vec{\sigma}(\vr_{\task})). \label{eq:zmu}
\end{equation}
Here vector $\vr_\task$ is a vector summarizing the encoding of a task $\task$. 
We propose a neural network module to learn  $\vr_\task$ from the sample observations $\datatr$.  The training observations of task $\task$ consist of unordered annotated data pairs $[(\x_{\task}^{tr}, \y_{\task}^{tr})]$. Permutation invariant is a desirable property for functions acting on sets.
As recommended by deep sets \cite{zaheer2017deep}, the authors proved any function acting on sets ${S}$ is permutation invariant if and only if it can be decomposed as $\rho(\sum_{\s\in S}\phi(\s))$ for suitable choice of transformations $\rho, \phi$. We follow such a design, and encode a task by encoding every pair of its observation in $\datatr$ through a neural network layer:
\begin{align}
&    \vr_{\task, j}  = \vg^{Enc}_{\vph}(\x_{\task, j}^{tr}, \y_{\task, j}^{tr}),\quad j = 1, \cdots, |\datatr|, \label{eq:encode} \\
&    \vr_\task  = \frac{1}{|\datatr|}\sum_{j=1}^{|\datatr|}\vr_{\task, j}.
    \label{eq:Rtask}
\end{align}
\eref{eq:Rtask} uses average function as aggregation operator to obtain the task embedding because it is able to remove the inductive bias due to different sizes of training set from $\vr_{\task}$. In \eref{eq:encode}, $\vg^{Enc}_{\vph}()$ is implemented as a MLP module with learnable parameter $\vph$.

We then approximate the intractable posterior distribution $q(\taskZ| \task)$ of $\taskZ$ as conditioned on the whole $\{\datatr, \datats\}$ (see  \sref{app:post_z}):
\begin{align}
     & q(\taskZ| \task) = q(\taskZ|\datatr, \datats) = \mathcal{N}(\vec{\mu}(\vr'_\task), \vec{\sigma}(\vr'_\task) ), \\  
     & \vr'_\task = \frac{1}{|\task|}\sum_{j=1}^{|\task|}\vr_{\task, j},\quad j = 1, \cdots, |\datatr| + |\datats|,
     \label{eq:postz}
\end{align}
where $|\task| = |\datatr| + |\datats|$ , $\vec{\mu}(\cdot)$ and $\vec{\sigma}(\cdot)$ are the same MLP modules we have in \eref{eq:zmu}.

%% file: files/04-1model.tex
\subsection{\method: Customizing Knowledge with $Z_{\task}$}

Now with the summary task representation $\taskZ$, we propose to use it to revise MAML into \method for heterogeneous meta-learning setup. We propose to tailor the global initialization $\goal$ to task-specific initialization $\goali^0$ for a task $\task$.

There exist many potential ways to use $\taskZ$ to tailor the global initialization $\goal$ to task-specific initialization $\goali^0$. We choose the following design. We assume, our target learning machine composes with a base learner and a task learner, like neural network models: $$\model = \vf_{\goal_c} ( \vf_{\goal_b}).$$ We assume the base learner's parameter is $\goal_b$, and its task learner's parameter is $\goal_c$ (for instance, the last linear layer before softmax for classification case). We can then rewrite $\goal=[\goal_b, \goal_c]$. We propose to only customize $\goal_c$ with $\taskZ$:
\begin{equation}
    \goali^0= \vg^{Gate}_{\vw}(\goal, \taskZ) =  [\goal_b, \sigma(\vec{w_1z}_{\task} + \vw_0)\odot\goal_c ],
    \label{eq:trans1}
\end{equation}
Here $\vec{z}_{\task}$ is sampled from the distribution $q(\taskZ|\task)$ during meta-training and from $p(\taskZ|\datatr)$ during meta-testing. $\sigma$ is the sigmoid function,  $\odot$ represents the element-wise multiplication, $\vw=[\vw_1, \vw_0]^T$ are learnable parameters.

Moreover, we design additional customized knowledge for task $\task$. The basic intuition is that the final prediction of a meta-learner depends on both model parameters and input representations. To increase the capacity of the task-specific knowledge, we propose to further propagate task representation $\taskZ$ into encoding augmented feature representations we denote as $\h_{\task}$. We concatenate $\h_{\task}$ with a sample's input representation $\x_{\task}$, and feed the combined vector $\hat{\x}_{\task}$ to our learning machine as its new input. 
\begin{equation}
    \h_{\task}^0 = \vg^{In}_{\vbeta}(\taskZ) = \vec{\beta}_1\vz_{\task} + \vec{\beta}_0,\quad \hat{\x}_{\task} = [\x_{\task}, \h_{\task}^0].
    \label{eq:trans2}
\end{equation}
Same as \eref{eq:trans1}, $\vz_{\task}$ is sampled from its distribution, $\vbeta = [\vbeta_1, \vbeta_0]$ are learnable parameters.

Now when facing a new task $\task$, a meta-model will first generate the task-specific knowledge that includes both augmented feature $\h_{\task}$ and task-specific parameter $\goali$. We denote the combined knowledge set for task $\task$ as: 
\begin{equation}
\biggoal=[\goali, \h_{\task}]. \label{eq:biggoal}
\end{equation}
This is the meta-knowledge we need to learn in \method. We note its initial values as $\biggoal^0 =[\goali^0, \h_{\task}^0]$ and fine-tuned values as $\biggoal^1=[\goali^1, \h_{\task}^1]$.

Aiming to learn the meta knowledge defined in \eref{eq:biggoal}, now we can write our objective (task likelihood) in \eref{eq:likeliz} into the following factorization: 
\begin{multline}
\begin{split}
\likeli  = \sum_{\biggoal^0, \biggoal^1, \taskZ} p(\datayite|\dataxite, \biggoal^1)p(\biggoal^1|\biggoal^0, \datatr) \\
p(\goali^0|\goal, \taskZ) p(\h_{\task}^0|\taskZ)p(\taskZ|\datatr).
\end{split}
\label{eq:pdist}
\end{multline}
This follows the Bayesian graph provided in Figure~\ref{fig:pgraph}.

\paragraph{Design Choices: } There exist many other possible probabilistic design besides Figure~\ref{fig:pgraph}. For instance, we can model every variable in the figure as a stochastic distribution and build a complicated hybrid framework. However, it will lead to excessive stochasticity and increase the potential of the underfitting issue especially in a limited data situation. Instead,
similar to $p(\biggoal^1|\biggoal^0, \datatr)$, we choose to model both $p(\h_{\task}^0|\taskZ)$ and $p(\goali^0|\goal, \taskZ)$ as deterministic (see \eref{eq:trans1} and \eref{eq:trans2}) that allow us to employ an amortized variational inference technique~\cite{ravi2019amortized}. %

Our design is different from recent probabilistic extensions of MAML~\cite{finn2018probabilistic, yoon2018bayesian}. They conduct inference on model parameters $\goali$ (initial value $\goali^0$ or fine-tuned value $\goali^1$). Our \method shifts the burden of variational inference to the task representation $\taskZ$, whose dimension is of multiple orders smaller than the size of model parameters.

%% file: files/04-2loss.tex
\subsection{\method: Update Rules}
\label{sec:update}

\begin{figure*}[t]
  \centering
  \includegraphics[width=0.95\textwidth]{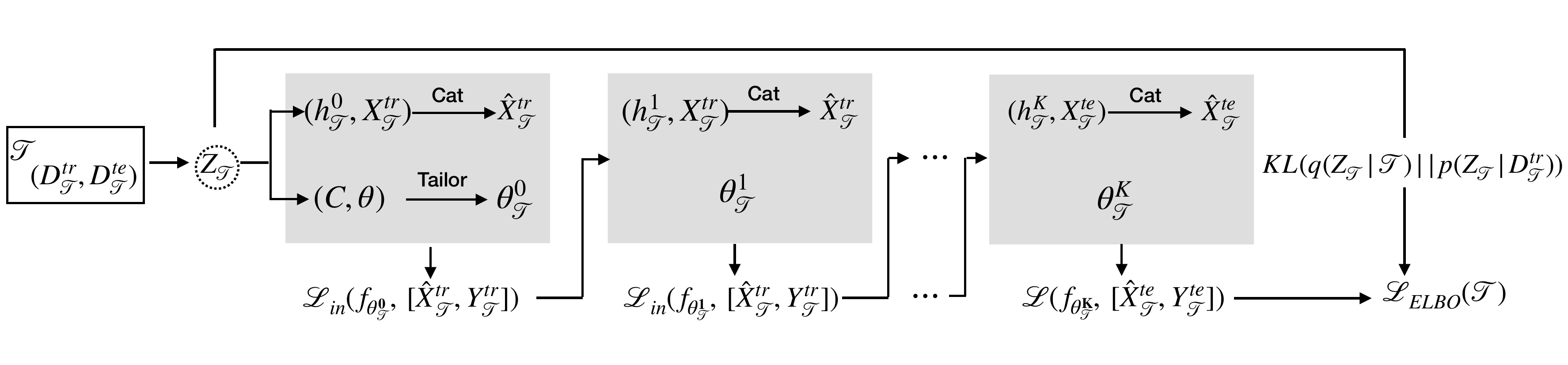}
  \caption{Iterative optimization process.  \label{fig:opt} In the inner loop, Starting from task-specific parameter initialization $\goali^0$ and augmented features $\h_{\task}^0$, their fine-tuned values $\goali^K, \h_\task^K$ are inferred by performing gradient descent on the training set $\datatr$ for $K$ iterations. }

\end{figure*}

\paragraph{Variational Objective:} To optimize the intractable likelihood as defined in \eref{eq:pdist}, we choose to maximize its evidence lower bound (a.k.a ELBO) instead:
 \begin{multline}
    \mathcal{L}_{ELBO}(\task) = \mathbf{E}_{\biggoal^1\sim q(\biggoal^1|\task)} \log p(\datayite|\dataxite, \biggoal^1)
    \\ - KL(q(\taskZ|\task)||p(\taskZ|\datatr).
\label{eq:LIB}
\end{multline}
During meta-training, we sample $m$ tasks and optimize the empirical average $\dfrac{1}{m}\sum\limits_{t=1}^m \mathcal{L}_{ELBO}(\task_t)$.

\paragraph{Update Rules:} Same as MAML, the optimization of the \method contains two loops: the inner loop and the outer loop. Figure~\ref{fig:opt} shows the iterative optimization process. In the inner loop, for the $j_{th}$ training data, we concatenate $\x_{\task, j}^{tr}$ with augmented feature $\h_{\task}^0$ to get augmented input vector $\hat{\vx}_{\task, j}^{tr}$. We
feed $\hat{\vx}_{\task, j}^{tr}$ into the learning machine $\modelnop$ whose parameter is $\goali^0$ to calculate the inner loss:
\begin{equation}
    \mathcal{L}_{in}(\task) = \frac{1}{|\datatr|}\sum_{j=1}^{|\datatr|}\mathcal{L}(\vf_{\goali^0}, [\hat{\x}_{\task, j}^{tr}, \y_{\task, j}^{tr}]).
    \label{eq:inner_loss}
\end{equation}

The inner loss is then used for updating $\goali^0$ and $h_{\task}^0$:
\begin{equation}
    \h_{\task}^{1} = \h_{\task}^{0} - \frac{\partial \mathcal{L}_{in}(\task)}{\partial \h_{\task}^{0}},\quad  \goali^{1} = \goali^{0} - \frac{\partial \mathcal{L}_{in}(\task)}{\partial \goali^{0}}.
    \label{eq:inner_update}
\end{equation}

Figure~\ref{fig:opt} shows we can optimize the inner loss for $K$ iterations to achieve a closer approximation for optimal values in \eref{eq:inner_loss}. In the outer loop, we maximize the approximated ELBO $\mathcal{L}_{ELBO}$ in \eref{eq:LIB} using a batch of $m$ tasks. The amortized variational technique allows us to conduct the sampling from $q(\biggoal^1|\task)$ by first sampling from $q(\taskZ|\task)$ and then apply deterministic transformation using \eref{eq:trans1} and \eref{eq:trans2}.

\paragraph{Algorithm of \method: }
We described the procedure of \method in the form of pseudo code as shown in Algorithm~\ref{alg:train}. Note, parameters of neural functions $\vec{\mu}(\cdot)$, $\vec{\sigma}(\cdot)$, $\vg^{Enc}_{\vph}()$, $\vg^{Gate}_{\vw}()$, and $\vg^{In}_{\vbeta}()$ are updated in the outer loop.

\paragraph{Theoretically Analysis of \method: } We also provide the second interpretation of our objective from information bottleneck perspective and prove they lead to exactly the same target. See \sref{app:derivation} for detailed proofs.

%% file: files/02related.tex
\subsection{Connecting to Related Work}
\label{related_work}

Optimization-based meta-learning methods facilitate the model's adaption to new tasks through global knowledge learned by the optimization process. Meta-LSTM~\cite{Ravi2017OptimizationAA} meta-learns the update rule with an RNN meta-learner. MAML~\cite{finn2017model} trains a global initialization close to the optimal value of every task. Leveraging diverse meta-knowledge further accelerates the learning process. In Meta-SGD~\cite{li2017meta}, the meta-knowledge consists of both initialization and learning rate. ALFA~\cite{baik2020meta} proposes to meta-learn both initialization and hyperparameter update module. 
Most methods assign the same global knowledge to every task that leads to sub-optimal solutions for heterogeneous settings. Besides, they are all deterministic and can only learn one solution for a new task.

Bayesian approaches are a long-standing discipline that incorporates uncertainty in modeling. Multiple recent works extend MAML into the Bayesian framework and recast meta-learning as the probabilistic framework~\cite{finn2018probabilistic, grant2018recasting, yoon2018bayesian, ravi2019amortized, garnelo2018neural}. PLATIPUS \cite{finn2018probabilistic} builds upon amortized variational inference and injects Gaussian noise into the gradient during the meta-testing time to learn a distribution over model parameters. LLAMA\cite{grant2018recasting} applies Laplace approximation for modeling the parameter distribution, but it requires the approximation of a high dimensional covariance matrix. These methods view model parameters (i.e. network weights and bias) as random variables and perform inference on them. It leads to significant challenges when working with complicated models and high-dimensional data. 

Our work also loosely connects to the "prototype meta-learning" ~\cite{triantafillou2019meta, snell2017prototypical}. These studies learn a prototype for every class we need to predict and the final prediction depends on the distances between instances and prototypes.  Amortized bayesian prototype meta-learning~ \cite{sun2021amortized} assumes a distribution over class prototypes. This design requires prior knowledge about the classes of tasks and only applies to the classification homogeneous-meta setup.

Another line of Bayesian meta-learning studies~\cite{garnelo2018neural, wang2020doubly,louizos2019functional, kim2018attentive} belongs to the neural approximators of the stochastic process family. They learn a prior for every task or further use a hierarchical model that learns the instance prior. However, these methods don't share knowledge across tasks. Table~\ref{table:model_comparison} compares related lines of works with ours.

%% file: files/05-2tables.tex
\begin{table}[h]
\caption{A summary of datasets, tasks  and their properties.}
\small
\centering
\resizebox{\columnwidth}{!}{%
\begin{tabular}{c|c|c| p{15mm}}
\toprule
Problems & Tasks & Heterogeneity & Ambiguity \\\midrule
\multirow{3}{*}{\shortstack{Regression}} & 
 2D regression & $+$ & $ 10 \to 40$  \\
 &Weather prediction & $++$ &  $10 \to 100$\\
&Image completion &  $+$ & $40 \to 784$
\\\cmidrule{1-4}
\multirow{3}{*}{\shortstack{Classification}}& PlainMulti classification & $+$ &  5way 5shot\\
& CelebA binary classification   &  & \multirow{2}*{2way 5shot} \\
& (see \sref{app:exp_binary}) & & \\
\bottomrule
\end{tabular}
}
\label{table:exp_summary}
\end{table}

\begin{table*}[h!]
\centering
\caption{Regression accuracy on 2D regression tasks.}
\begin{tabular}{c|cccccc}
\hline
Model & MAML       & MetaSGD   & BMAML      & MMAML      & ARML       & \method  \\ \hline
MSE   & $2.29\pm 0.16$ & $2.91\pm 0.23$ & $1.65\pm 0.10$ & $0.52\pm 0.04$ & $0.44\pm 0.03$ &  $\mathbf{0.37\pm 0.04}$      \\ \hline
\end{tabular}
\label{table:2Dregression}
\end{table*}

\begin{table*}[h!]
\caption{10-Shot temperature prediction.}
\centering
\begin{tabular}{c|ccccc}
\hline
Model & MAML   & MetaSGD & \method  & \method(w/o aug) & \method(w/o tailor) \\ \hline
MSE   & $141.43 \pm 9.33\%$ & $291.42 \pm 14.89\%$ &  $\mathbf{86.56 \pm 4.89\%}$  &   $100.27 \pm 5.87\%$ & $106.37 \pm 5.77\% $    \\\hline
\end{tabular}
\label{table:weather_reg}
\end{table*}

\begin{figure*}[h!]
    \centering
    \makebox[\textwidth][c]{\includegraphics[width=0.95\textwidth]{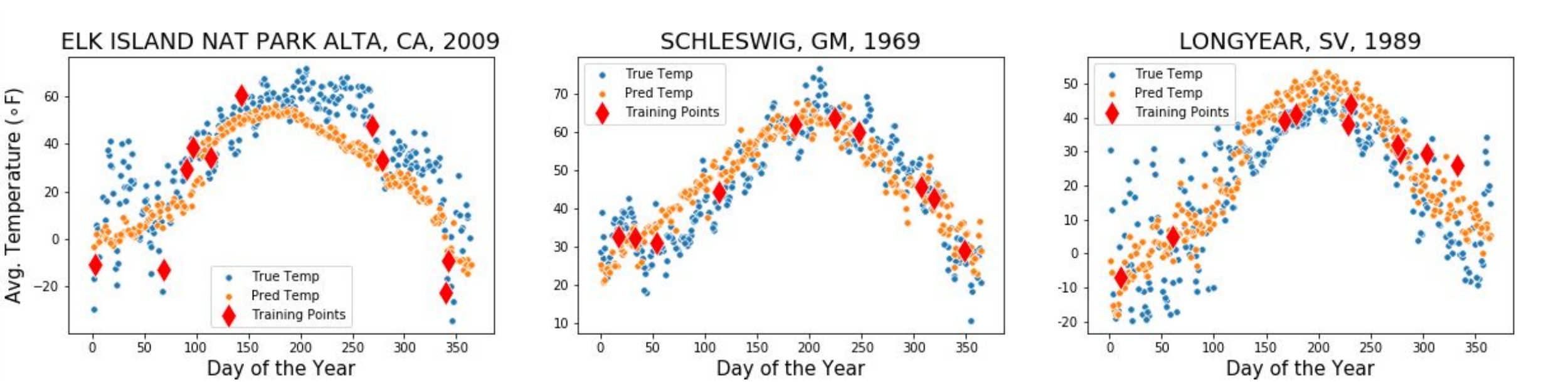}}
    \caption{A visualization of trained \method on the NOAA-GSOD temperature prediction task. The model is given $10$ training points (red) and predicts the remaining days of the year (orange). The true temperatures are shown in blue.}
    \label{fig:temp_preds}
\end{figure*}

%% file: files/05-1exp.tex
\section{Experiments}

We apply \method to both few-shot regression and classification to demonstrate its effectiveness on both heterogeneous and ambiguous tasks. In regression, we evaluate \method in a variety of domains including 2D curve fitting, whose tasks show both heterogeneity and ambiguity, and two real-world tasks including image completion and weather prediction. We also study two few-shot classification problems. Tasks from the Plain-Multi dataset are heterogeneous while CelebA classification uses ambiguous decision rules.   
A summary of experiment design, datasets, and their properties is shown in Table~\ref{table:exp_summary}. The number of $+$ represents the significance level of the challenge.

\subsection{2D Regression}
\textbf{Setup.} For 2D regression, we follow the similar setting as~\cite{yao2020automated}, where $|P(\task)|=6$. The meta distribution $P(\task)$ consists of 6 function families including \textit{sinusoids, straight line, quadratic, cubic, quadratic surface}, and \textit{ripple} functions. To increase ambiguity, we perturb the output by adding a Gaussian noise whose standard deviation is 0.3. During meta-training, every task is uniformly randomly sampled from one of them, the size of the training set $|\datatr|=10$. A detailed description of the setup and model architecture is available in appendix~(see \sref{app:exp_setup}).

\textbf{Baselines, results, and analysis.}  We have two types of baselines: (1) meta-learning methods designed for homogeneous tasks: MAML~\cite{finn2017model} and MetaSGD~\cite{li2017meta}. (2) Bayesian meta-learning method: Bayesian MAML~\cite{yoon2018bayesian}, which conducts inference on a large number of model parameters. (3) Meta-learning methods designed for heterogeneous tasks including MMAML\cite{vuorio2019multimodal} and ARML\cite{yao2020automated}. We train our model on around $10,000$ tasks and evaluate it on over $1,000$ new sampled tasks. The results are summarized in Table~\ref{table:2Dregression}. We showcase fitting curves in appendix~(see Figure~\ref{fig:2DReg_10shot}). Even though we fix the size of the training set and noise level for every task during meta-training, during meta-testing, they are flexible and can be changed. To increase the ambiguity for every test task, we vary the number of available annotated data in the training set and noise level. More analysis visualization results can be found in appendix~\ref{app:exp_setup}. 

As in Figure~\ref{fig:2DReg_10shot}, all sampled solutions will be close to the groundtruth if tasks are less uncertain. On the other hand, the figures in appendix~\ref{app:exp_setup}  show that as tasks become more ambiguous, due to fewer annotated training data or larger noise, the sampled solutions tend to span wider space. 

\subsection{Temperature Prediction}
\textbf{Setup.} Next, we evaluate the model in a challenging regression problem using real-world data. The NOAA Global Surface Summary of the Day (GSOD) dataset contains daily weather data from thousands of stations around the world. Each task is created by sampling data points from (station, year) pairs. The model takes in the current day of the year along with $15$ weather features such as wind speed, station elevation, precipitation, fog, air pressure, etc. It then learns to predict the average temperature in Fahrenheit on that day. We remove important information like the weather station number, name, latitude, and longitude. Hiding the station information in this way creates a highly heterogeneous problem where each station generates its own task distribution. The model sees $10$ days of labeled temperature data before predicting the temperature on $100$ test days. More technical details can be found in appendix \ref{app:exp_setup}.

\textbf{Results and analysis.} After $100$ epochs of training on approximately $42,000$ unique (station, year) tasks, we evaluate the model on a test set of $1,000$ (station, year) pairs. The results are summarized in Table \ref{table:weather_reg}. The MSE error of MAML is close to double that of \method. MetaSGD, designed for homogeneous meta-learning, achieves low accuracy because the globally learned learning rate will hurt the model's generalization ability on unseen tasks from different distributions. It is consistent with our assumption that incorporating task-specific knowledge into the model can help solve the task-heterogeneous challenge.

\begin{table*}[th!]
\caption{5-way 5-shot classification accuracy with 95\% confidence interval on Plain-Multi dataset.}
\small
\begin{center}
\begin{tabular}{l|l|cccc}
\toprule
Settings & Algorithms & Data: Bird & Data: Texture & Data: Aircraft & Data: Fungi \\\midrule
\multirow{7}{*}{\shortstack{5-way\\5-shot}} & 
 MAML & $68.52\pm0.79\%$ & $44.56\pm0.68\%$ & $66.18\pm 0.71\%$ & $51.85\pm0.85\%$ \\
&MetaSGD & $67.87\pm0.74\%$ & $45.49\pm0.68\%$ & $66.84\pm0.70\%$ & $52.51\pm0.81\%$ \\\cmidrule{2-6}
& BMAML & $69.01\pm 0.74\%$ & $46.06\pm 0.69\%$ & $65.74\pm 0.67\%$ & $52.43\pm 0.84\%$ \\
& MMAML & $70.49\pm0.76\%$ & $45.89\pm0.69\%$ & $67.31\pm0.68\%$ & $53.96\pm0.82\%$ \\
&  HSML & $\mathbf{71.68\pm 0.73\%}$ & $\mathbf{48.08\pm 0.69\%}$ & $\mathbf{73.49\pm 0.68\%}$ & $\mathbf{56.32\pm 0.80\%}$ \\\cmidrule{2-6}
& \method  & $\mathbf{72.49 \pm 0.53\%}$  & $46.51 \pm 0.42\%$ & $\mathbf{72.64 \pm 0.44\%}$ & $\mathbf{55.29 \pm 0.57\%}$ \\
& \method(w/o aug) & $71.49 \pm 0.55\%$  & $\mathbf{47.17 \pm 0.44\%}$ & $71.62 \pm 0.43\%$ & $54.91 \pm 0.56\%$ \\
& \method(w/o tailor) & $71.48 \pm 0.55\%$ & $46.07 \pm 0.40\%$ & $70.46 \pm 0.44\% $& $54.59 \pm 0.56\%$\\
\bottomrule
\end{tabular}
\end{center}
\label{tab:plainmulti_res}
\end{table*}

\subsection{Image Completion}
\textbf{Setup.} We also apply our method to image completion tasks. In image completion, the meta distribution $p(\task) = \{\text{MNSIT}, \text{FMNIST}, \text{KMNIST}\}$. Every task contains one image of size $28\times 28$ sampled randomly from one of three distributions. In meta-training, $40$ pixels are observed for every image, thus, $|\datatr|=40$. We use coordinates as inputs and pixel value as the target variable. Detailed architecture can be found in appendix~\ref{app:exp_setup}.
\begin{figure*}[t]
    \centering 
    \includegraphics[width=0.9\linewidth]{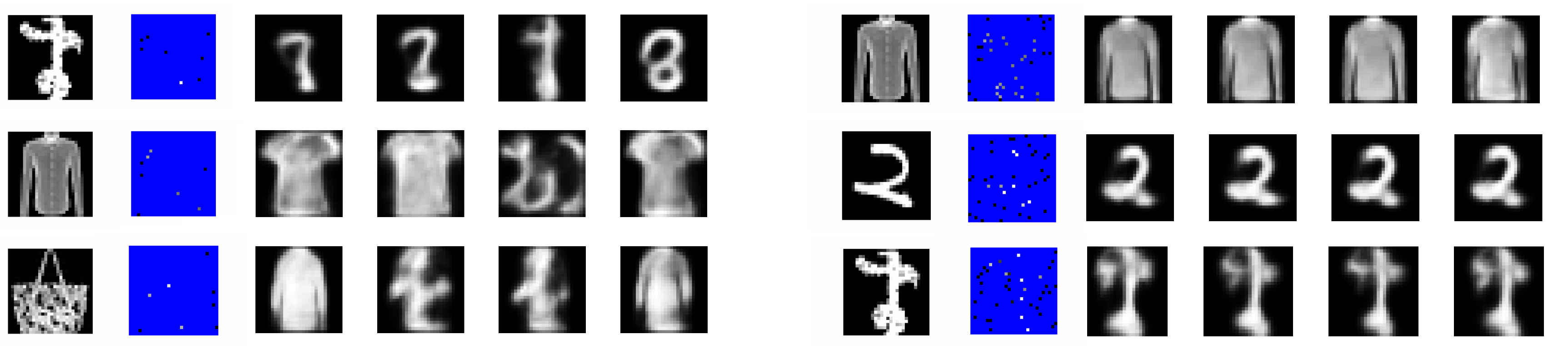}
\caption{Visualization of completed images. First column contains  original images, second column shows the observations which contains only $8$ annotated pixels(left) and 40 annotated pixels(right). The unobserved pixels have been coloured blue for better clarity. The remaining columns correspond to $4$ different samples given the context points. } 
\label{fig:ImgCompletion}
\end{figure*}

\textbf{Baselines, results and analysis.} Image completion with limited given pixels is a benchmark task for Neural processes~\cite{garnelo2018conditional,garnelo2018neural}. Thus, we compare our proposed \method with neural processes(NP)~\cite{garnelo2018neural} and conditional neural processes~\cite{garnelo2018conditional} which is viewed as deterministic neural processes. Similar to CNP, we also recast our model into the deterministic framework, where the task representation $\taskZ$ is modeled as a fixed-dimension vector learned from the training set $\datatr$ only. The numerical comparison is shown in Table~\ref{table:img_comp}. \method achieves higher completion precision compared with NP and CNP. We skipped the variance for all methods because the difference is insignificant and is close to $1\mathrm{e}{-5}$.

There is a large amount of ambiguity surrounding the completed images. Given limited observed pixels, multiple potential images are lying behind, especially for gray images. Uncertainty arises on three levels: inter-class level, inter-distribution level, and cross-distribution level. \method can increase the opportunity of capturing more potential truths by learning a distribution of possibilities rather than a unique mapping. We visualize observations and their completions in Figure~\ref{fig:ImgCompletion}. Our set operations allows us to learn from any size of the training set during meta-testing. Thus, as more pixels are observed during meta-testing, the task is less ambiguous. Therefore, the completed images from different models stay close to the original image. 
\begin{table}[h!]
 
    \centering
\caption{Image completion accuracy.   \label{table:img_comp}}
\resizebox{0.9\columnwidth}{!}{%
\begin{tabular}{c|cccc}
\hline
Model & NP    & CNP   & \method (deter) & \method  \\ \hline
BCE   & $0.302$ & $0.358  $ & $0.272 $           & $\mathbf{0.268}$    \\ \hline
\end{tabular}
}
\end{table}

\subsection{Heterogeneous Classification}
\textbf{Setup and baselines.} N-way K-shot classification is a popular setup in few-shot meta-learning~\cite{chen2019closer, ren2018meta, vinyals2016matching}. The training set of every task consists of $N$ classes with $K$ labeled data in each class. We apply our proposed \method on the benchmark heterogeneous meta-learning dataset: Plain-Multi, proposed in~\cite{yao2019hierarchically}. The meta-distributions consists of four datasets and every task is sampled uniformly randomly from one of them. Following the benchmark architecture, the feature learner contains four convolutional blocks. The input $\vec{x}$ is feed into two convolutional blocks with 6 channels, then the output is appended with the target variable and passed into a two-layer MLP module to model the mean and variance of $\taskZ$. We compare to MAML~\cite{finn2017model}, MetaSGD~\cite{li2017meta}, MMAML~\cite{vuorio2019multimodal}, HSML~\cite{yao2019hierarchically}, and probabilistic method BMAML~\cite{yoon2018bayesian}.

\textbf{Results and analysis.} Trained on over $50,000$ tasks, the model is evaluated on $1,000$ tasks for each dataset and the results are summarized in Table~\ref{tab:plainmulti_res}. The most relevant method is MMAML. It learns a deterministic task embedding with an RNN module and encodes all parameters in both base learner $\vf_{\goal_b}$ and task learner $\vf_{\goal_c}$. Our method outperforms it on every dataset. Also, the probabilistic framework enables us to achieve consistently low variance. HSML requires the prior knowledge about number of clusters, which plays an important role with respect to the final accuracy.

\subsection{Ablation Studies.} 
Facing a task, the initial state of the knowledge set includes both tailored initialization and augmented feature. To better investigate the contribution of each component, we perform ablation experiments on both temperature prediction and PlainMulti classification. The results are shown in both Table \ref{table:weather_reg} and Table \ref{tab:plainmulti_res}. Both two types of task-specific knowledge exhibit the performance improvement over the baselines, and they together give the best performance.

%% file: files/08conclude.tex
\section{Conclusion}
Task heterogeneity and task ambiguity are two critical challenges in meta-learning. Most meta-learning methods assign the same initialization to every task and fail to handle task heterogeneity. They also disregard the task ambiguity issue and learn one solution for every task. \method encodes tasks using NN-based stochastic task module plus set-based operation for permutation-invariance. This stochastic task design allows for customizing global knowledge with learned stochastic task distribution. We further convert latent task encodings to augmented features to improve the interaction between model parameters and input variables. The probabilistic framework allows us to learn a distribution of solutions for ambiguous tasks and recover more potential task identities. Empirically, we design extensive experiments on regression and classification problems and show that \method provides an efficient way to learn from diverse and ambiguous tasks. We leave the challenge to handle domain generalization during meta-testing to future work.

%% file: files/99-related-table.tex
\subsection{Model Comparison.}
\begin{table*}[h]
\caption{Model comparison table. HoMAMLs are MAMLs designed for task homogeneity, and HeMAMLs are for heterogeneity. NPs describe methods in Neural Processes family. PMAMLs mean probabilistic extensions of MAML. Aug feature represents the augmented features.}
\small
\begin{center}
\begin{tabular}{c|c|cccc}
\toprule
Category & Tasks & Knowledge Set & Tailoring & Sampling & Inference on \\\midrule
\multirow{2}{*}{\shortstack{HoMAMLs}} & 
 MAML~\cite{finn2017model} & Initialization & & &  \\
 &MetaSGD~\cite{li2017meta} & Initialization$+$lr & & & \\\cmidrule{1-6}
\multirow{2}{*}{\shortstack{HeMAMLs}}& MMAML~\cite{vuorio2019multimodal}& Initialization &\checkmark & & \\ 
&HSML~\cite{yao2019hierarchically} & Initialization &\checkmark & &  \\\cmidrule{1-6}
\multirow{2}{*}{\shortstack{NPs}} & 
 NP~\cite{garnelo2018neural} & Aug feature  & & \checkmark & Representation \\
 &CNP~\cite{garnelo2018conditional} & Aug feature & &   & \\\cmidrule{1-6}
 \multirow{3}{*}{\shortstack{PMAMLs}}& BMAML~\cite{yoon2018bayesian}& Initialization  &  &\checkmark & Parameters\\ 
&PLATIPUS~\cite{finn2018probabilistic} &Initialization & & \checkmark  & Parameters \\
&\method &Initialization$+$Aug feature & \checkmark & \checkmark & Representation\\

\bottomrule
\end{tabular}
\end{center}
\label{table:model_comparison}
\end{table*}

%% file: files/99-theory.tex
\subsection{Approximation for posterior distribution $q(Z_{\task})$.}
\label{app:post_z}
Given the training set $\datatr$ of a task $\task$, the stochastic task variable $\taskZ$ is supposed to infer its posterior distribution conditioned on $\datatr$ only, specifically, we have the true posterior:
\begin{equation}
p(\taskZ|\task) = \dfrac{p(\taskZ|\datatr) p(\datayite|\taskZ, \dataxite, \datatr)}{p(\task)}
\end{equation}
the empirical distribution $p(\task)$ is only known in the form of $\{(\datatr, \datats)\}$ pairs. Thus, the true posterior distribution is intractable. Based on our design, we suppose the prior distribution $p(\taskZ|\datatr)$ is a multivariate Gaussian distribution, whose mean and variance is the output of a set operator acting on $(\dataxitr, \datayitr)$ pairs. To ensure the posterior stays close to the prior, also the posterior is derived from $(\datatr, \datats)$, we approximate it with the output of the same set operator acting on both $(\dataxitr, \datayitr)$ and $(\dataxite, \datayite)$ pairs.

\subsection{Derivation of ELBO  approximation as Variational Information Bottleneck Objective}
\label{app:derivation}

For task $\task$, our fine-tuned task-specific knowledge set $\biggoal^1$ contains two variables: model parameters $\goali^1$ and augmented features $\h_{\task}^1$. Given task inputs $\X_{\task}=[\dataxitr, \dataxite]$, we are seeking a task-specific knowledge set that is maximally informative of test target $\datayite$, while being mostly compressive of training target $\datayitr$~\cite{titsias2020information, tishby2000information}. Correspondingly, we would like to maximize the conditional mutual information $I(\datayite; \biggoal^1| \X_\task)$ and minimize $I(\datayitr; \biggoal^1 | \X_\task)$. The information bottleneck objective is:
\begin{equation}
\mathcal{L}_{IB}(\task) = I(\datayite; \biggoal^1| \X_\task) - \beta I(\datayitr; \biggoal^1 | \X_\task).
\label{eq:IB}
\end{equation}

We show the following lemma in appendix~\ref{app:derivation}:
\begin{lemma}
Given a task $\task$, maximizing the information bottleneck loss $\mathcal{L}_{IB}$ defined in \eqref{eq:IB} is equivalent to maximizing the weighted ELBO :
 \begin{align}
    \mathcal{L}_{wELBO}(\task) = \mathbf{E}_{\biggoal^1\sim q(\biggoal^1|\task)} \log p(\datayite|\biggoal^1,\dataxite)- \beta KL(q(\taskZ|\task)||p(\taskZ|\datatr).
\label{eq:weightLIB}
\end{align}
\end{lemma}

\begin{proof}To lower bound IB objective defined in \eref{eq:IB}, we derive the lower bound for first term $I(\datayite; \biggoal^1 | \X_{\task})$ and upper bound for second term $I(\datayitr; \biggoal^1 | \X_{\task})$. 
Further, we assume a distribution $q(\datayite, \biggoal^1|\X_{\task})$ as a variational approximation of the true distribution $p(\datayite, \biggoal^1|\X_{\task})$. 
\begin{equation}
\begin{split}
I(\datayite, \biggoal^1|\X_{\task})
&=\int p(\X_{\task})\left[\int q(\datayite, \biggoal^1|\X_{\task}) \log\dfrac{q(\datayite, \biggoal^1|\X_{\task})}{p(\datayite) q(\biggoal^1|X)}d\datayite d\biggoal^1\right]d\X_{\task}\\
&= \int p(\X_{\task})\left[\int q(\datayite, \biggoal^1 | \X_{\task}) \log\dfrac{q(\datayite| \biggoal^1, \X_{\task})}{p(\datayite)}d\datayite d\biggoal^1\right]d\X_{\task}
\end{split}
\label{eq:IB_term1}
\end{equation}

\begin{equation}
\begin{split}
    q(\biggoal^1|\X_{\task})&=\int q(\biggoal^1 | \datayitr, \X_{\task}) p(\datayitr ) d\datayitr \\
    &= \int q(\biggoal^1 | \datayitr , \X_{\task}) p(\datayitr ,\datayite ) d\datayitr d\datayite 
    \end{split}
\end{equation}

\begin{equation}
\begin{split}
    q(\datayite , \biggoal^1|\X_{\task})&=\int q(\biggoal^1, \datayitr ,\datayite   |  \X_{\task})  d\datayitr \\
    &=\int q(\biggoal^1, | \datayitr ,\datayite   ,  \X_{\task}) p(\datayitr ,\datayite   | \X_{\task})  d\datayitr \\
    &=\int q(\biggoal^1, | \datayitr   ,  \X_{\task}) p(\datayitr ,\datayite   | \X_{\task})  d\datayitr \\
   \end{split}
\end{equation}
The last part follows  from the fact that $\biggoal^1$ is independent of $\datayite $ given $[\X_{\task}, \datayitr ]$. Putting this together: 

\begin{equation}
q(\datayite |\biggoal^1, \X_{\task}) = \dfrac{\int p(\datayite ,\datayitr )q(\biggoal^1 | \datayitr , \X_{\task})d\datayitr }{\int p(\datayite ,\datayitr )q(\biggoal^1 | \datayitr , \X_{\task})d\datayitr d\datayite }
\end{equation}

However, the above conditional distribution $q(\datayite |\biggoal^1, \X_{\task})$ is intractable due to the unknown data distribution $p(\datayite ,\datayitr )$. To derive the upper bound, we introduce a variational approximation $ p_{\theta}(\datayite |\biggoal^1, \X_{\task})$ for $q(\datayite |\biggoal^1, \X_{\task})$. 

Take it into the \eref{eq:IB_term1}, we have:

\begin{equation}
\begin{split}
I(\datayite , \biggoal^1|\X_{\task}) &= \int p(\X_{\task}) \left[\int q(\datayite , \biggoal^1|\X_{\task})\log\dfrac{p_{\theta}(\datayite |\biggoal^1, \X_{\task})q(\datayite | \biggoal^1, \X_{\task})}{p_{\theta}(\datayite |\biggoal^1, \X_{\task})p(\datayite )}d\datayite d\biggoal^1\right]d\X_{\task}\\
&\geq \int p(\X_{\task}) \left[\int q(\datayite , \biggoal^1| \X_{\task})\log\dfrac{p_{\theta}(\datayite |\biggoal^1, \X_{\task})}{p(\datayite )}d\datayite d\biggoal^1\right]d\X_{\task}\\
&= \int p(\X_{\task}) \left[\int q(\datayite , \biggoal^1|\X_{\task})\log p_{\theta}(\datayite |\biggoal^1, \X_{\task})d\datayite d\biggoal^1\right]d\X_{\task} + C\\
&= \int q(\datayite , \biggoal^1, \X_{\task})\log p_{\theta}(\datayite |\biggoal^1, \X_{\task})d\datayite d\biggoal^1d\X_{\task} + C
\end{split}
\label{eq:IB_term1_app1}
\end{equation}

In the above equation, we use  $KL(q(\datayite |\biggoal^1,\X_{\task})|| p_{\theta}(\datayite |\biggoal^1, \X_{\task})) \geq 0$ in the second step.

The second term is irrelevant to our objective so we can treat it as a constant. Note that:
\begin{equation}
 q(\datayite , \biggoal^1, \X_{\task})
 =\int q(\biggoal^1|\datayitr , \X_{\task})p(\datayitr, \datayite|\X_{\task})p(\X_{\task})d\datayitr 
\end{equation}
Thus, an unbiased estimation of the first term is:
\begin{equation}
I(\datayite , \biggoal^1|\X_{\task}) \geq  \int  q(\biggoal^1 | \datayitr, \X_{\task})\log p_{\theta}(\datayite |\biggoal^1, \X_{\task}) d\biggoal^1.
\label{eq:upper_bound}
\end{equation}

We derive the upper bound for second term:
\begin{equation}
\begin{split}
I(\datayitr , \biggoal^1|\X_{\task}) 
&=\int p(\X_{\task})\left[\int q(\datayitr , \biggoal^1|\X_{\task}) \log\dfrac{q(\datayitr , \biggoal^1|\X_{\task})}{p(\datayitr ) q(\biggoal^1|\X_{\task})} d\datayitr d\biggoal^1\right]d\X_{\task}\\
&= \int p(\X_{\task})\left[\int q(\datayitr , \biggoal^1 | \X_{\task}) \log\dfrac{q(\biggoal^1 | \datayitr , \X_{\task})}{q(\biggoal^1|\X_{\task})}d\datayitr d\biggoal^1\right]d\X_{\task}
\end{split}
\label{eq:IB_term2}
\end{equation}
The denominator $q(\biggoal^1|\X_{\task}) = \int q(\biggoal^1|\datayitr , \X_{\task})p(\datayitr )d\datayitr $ is intractable for unknown $p(\datayitr )$. We  approximate it with $p_{\theta}(\biggoal^1|\X_{\task})$. With similar derivation, the second term is upper bounded by:
\begin{equation}
 I(\datayitr , \biggoal^1|\X_{\task}) \leq  \int q(\biggoal^1|\datayitr , \X_{\task})p(\datayitr , \X_{\task})\log\dfrac{q(\biggoal^1 | \datayitr , \X_{\task} )}{p_{\theta}(\biggoal^1|\X_{\task})}d\datayite d\datayitr d\biggoal^1.
\end{equation}

Similarly, its unbiased estimation is given as:
\begin{equation}
 I(\datayitr , \biggoal^1|\X_{\task}) \leq  \int q(\biggoal^1|\datayitr , \X_{\task})\log\dfrac{q(\biggoal^1 | \datayitr , \X_{\task} )}{p_{\theta}(\biggoal^1|\X_{\task})}d\biggoal^1.
\end{equation}

Combining two terms, we get the total unbiased estimation of the IB loss:
\begin{equation}
    L_{IB} = \mathbf{E}_{q(\biggoal^1 | \datayitr , \X_{\task})}\log p_{\theta}(\datayite |\biggoal^1, \X_{\task})  - \beta KL(q(\biggoal^1|\datayitr , \X_{\task})||p_{\theta}(\biggoal^1|\X_{\task})).
\label{eq:LIB_unbias}
\end{equation}

To incorporate target information, we inject the target variable $\datayite $ into posterior and $\datayitr $ into prior, and get the new approximation:
\begin{equation}
    L_{IB} = \mathbf{E}_{q(\biggoal^1 | \task)}\log p_{\theta}(\datayite |\biggoal^1, \X_{\task})
    - \beta KL(q(\biggoal^1|\task)||p_{\theta}(\biggoal^1|\datayitr , \X_{\task})).
\label{eq:LIB_unbias_np}
\end{equation}

Since $\goali^0=\vg^{Gate}_{\vw}(\goal, \taskZ), \h_{\task}^0=\vg^{Gate}_{\vbeta}(\taskZ)$, where $\vg^{Gate}_{\vw}, \vg^{Gate}_{\vbeta}$ are both deterministic and invertible mappings of $\taskZ$, we have $ p(\goali^0|\goal) = \delta(\goali^0=\vg^{Gate}_{\vw}(\taskZ, \goal)), p(\h_{\task}^0|\taskZ) = \delta(\h_{\task}^0=\vg^{Gate}_{\vbeta}(\taskZ))$. Moreover, $\h_{\task}^0, \goali^0$ are conditionally independent given $\taskZ$. Similarly, $\h_{\task}^1, \goali^1$ are deterministic function of $\h_{\task}^0$ and $\goali^0$. Thus, the second term in \eref{eq:LIB_unbias_np} can be replaced with the divergence between the posterior and prior distribution of $\taskZ$, i.e. $KL(q(\taskZ|\task)||p(\taskZ|\datayitr , \dataxitr))$.

We know look into the log likelihood term in \eref{eq:LIB_unbias}. Since the transitions $\taskZ \to \goali^0 \to \goali^1$ and $\taskZ \to \h_{\task}^0 \to \h_{\task}^1$ are deterministic:
\begin{equation}
\begin{split}
    \goali^1 = \goali^{0} -\nabla_{\goal} \mathcal{L}(\vf_{\goali^0}, \h_{\task}^0, \datatr)),\quad  \goali^0 = \vg^{Gate}_{\vw}(\goal, \vz), \quad \vz\sim q(\taskZ|\task)\\
    \h_{\task}^1 = \h_{\task}^0 - \nabla_{\h}\mathcal{L}(\vf_{\goali^0}, \h_{\task}^0, \datatr)),\quad \h_{\task}^0 = \vg^{Gate}_{\vbeta}(\vz).
    \end{split}
\label{eq:g_theta_post}
\end{equation}

According to the analysis, the approximation to be optimized is:
\begin{equation}
    L_{app} = \mathbf{E}_{\biggoal^1\sim q(\biggoal|\task)} \log p_{\theta}(\datayite | \goali^1, X^{te})
    - \beta KL(q(\taskZ|\task)||p(\taskZ|\datatr)).
\label{eq:LIB_app2}
\end{equation}
\end{proof}

%% file: files/99-moreExp.tex
\subsection{Ambiguous Binary Classification Results.}
\label{app:exp_binary}
\textbf{Task design.} In classification, task ambiguity is common when annotated data are limited. Images can share many attributes, and various combinations of them can be used for final decision-making. We evaluate our method on the ambiguous classification benchmark proposed in \cite{finn2018probabilistic}. The CelebA dataset contains cropped images of celebrity faces and a list of attributes that describe their appearance.  We split these attributes into training, validation, and test sets. During meta-training, we randomly sample two training attributes and form the positive class of images that share them. The negative class is formed by sampling the same number of images containing neither attribute. During meta-testing, training set images share three attributes. We construct three test sets by choosing two of the three attributes to define the positive class. The model learns to apply two attributes for decision making, but there are three combinations of two attributes for classification. Thus the task is ambiguous. We sample models from our distribution of solutions and assign them to the three test sets based on the loss values. If all test sets are covered with at least one model, the method can effectively discover all potential decision rules. The cover number is calculated as the average number of test sets that are covered. The coverage number for a deterministic method is $1$. As Table~\ref{table:amb_clf} shows, our method can 1) achieve better accuracy, 2) reach lower NLL, and 3) discover more decision rules compared to MAML.

\begin{table}[h!]
\centering
\caption{5-Shot Ambiguous Binary Classification.\label{table:amb_clf}}
\resizebox{0.45\textwidth}{!}{%
\begin{tabular}{c|ccc}
\hline
Model & Accuracy    & Coverage number & NLL  \\ \hline
MAML   & 77.924  & 1.00 &     0.454            \\
\method  & 79.698 & 1.13  &    0.439            \\\hline
\end{tabular}}
\end{table}
\subsection{Experiment setup.}
\label{app:exp_setup}

\textbf{2D Regression setup.} Meta distribution $\mathcal{T}$ contains 6 function families. Input $X = [x_1,x_2]\sim U(0.0, 5.0)$. The value for $x_2$ is fixed as 1 if only $x_1$ is used. For \textit{sinusoids} families : $y = a sin(wx_1 + b)+\epsilon$, where $a \sim U[0.1, 5.0], b \sim U[0, 2\pi], w\sim  U[0.8, 1.2]$; for \textit{line} families: $y = ax_1 + b+\epsilon$, where $a\sim  U[-3.0,3.0], b\sim U[-3.0,3.0]$;  for \textit{quadratic curves}: $y = ax_1^2 + bx_1 + c+\epsilon$, where $a\sim U[-0.2, 0.2], b\sim U[-2.0, 2.0], c\sim U[-3.0, 3.0]$; for cubic curves: $y = ax_1^3 + bx_1^2 + cx_1 + d+\epsilon$, where $a \sim U [-0.1, 0.1], b\sim U [-0.2, 0.2], c\sim U [-2.0, 2.0], d\sim U [-3.0, 3.0]$; for \textit{quadratic surface}: $y = ax_1^2 + bx_2^2+\epsilon$, where $a\sim U[-1.0, 1.0], b\sim U[-1.0, 1.0]$;  for \textit{ripple}: $y = sin(-a(x_1^2 + x_2^2)) + b+\epsilon$, where $a\sim U[-0.2,0.2], b\sim U[-3.0,3.0]$. 

\textbf{Model architecture for 2D regression.} We adopt the same base model as in~\cite{yao2020automated, finn2017model}, it contains 2 linear layer with 40 neurons followed by ReLU function.  For the task representative module, we use 2 linear layers with 80 neurons.

\textbf{Visualization for 2D regression.} See Figure~\ref{fig:2DReg_10shot}.
\begin{figure*}[thb]
  \centering 
    \begin{minipage}[b]{.24\textwidth} 
      \centering 
      \includegraphics[width=\linewidth]{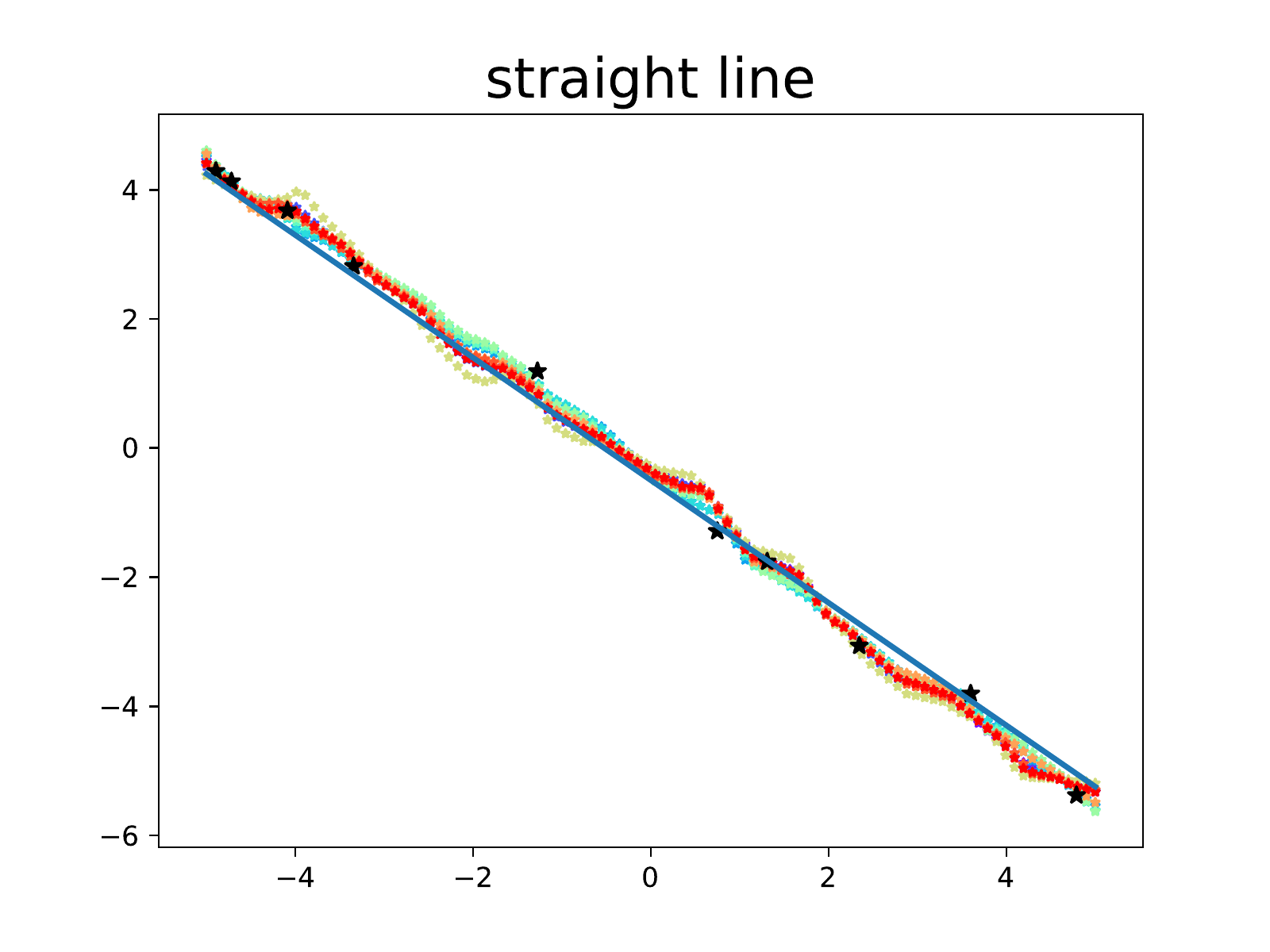}
    \end{minipage} 
    \begin{minipage}[b]{0.24\textwidth} 
      \centering 
      \includegraphics[width=1\linewidth]{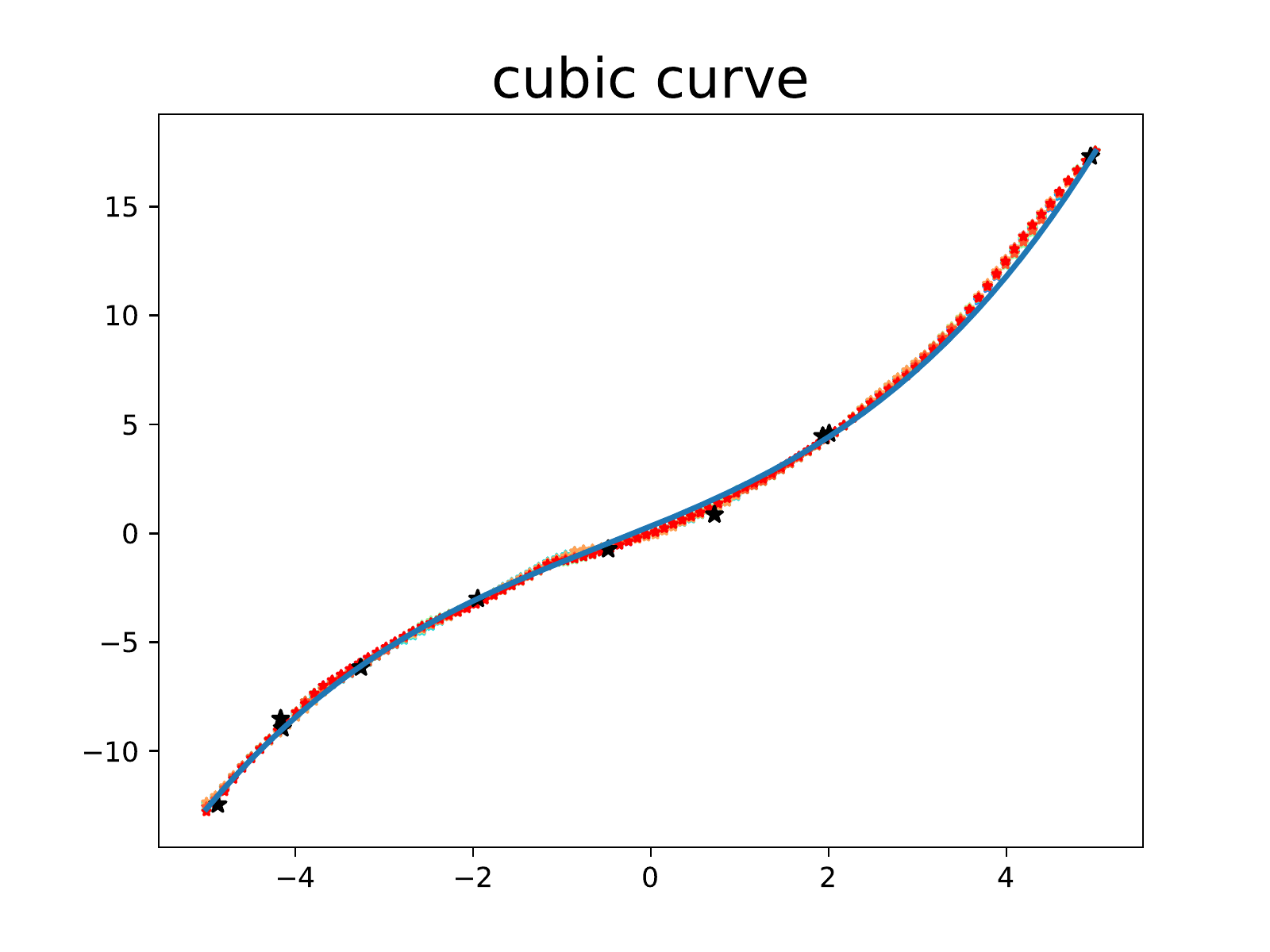}
    \end{minipage} 
    \begin{minipage}[b]{0.24\textwidth} 
      \centering 
      \includegraphics[width=1\linewidth]{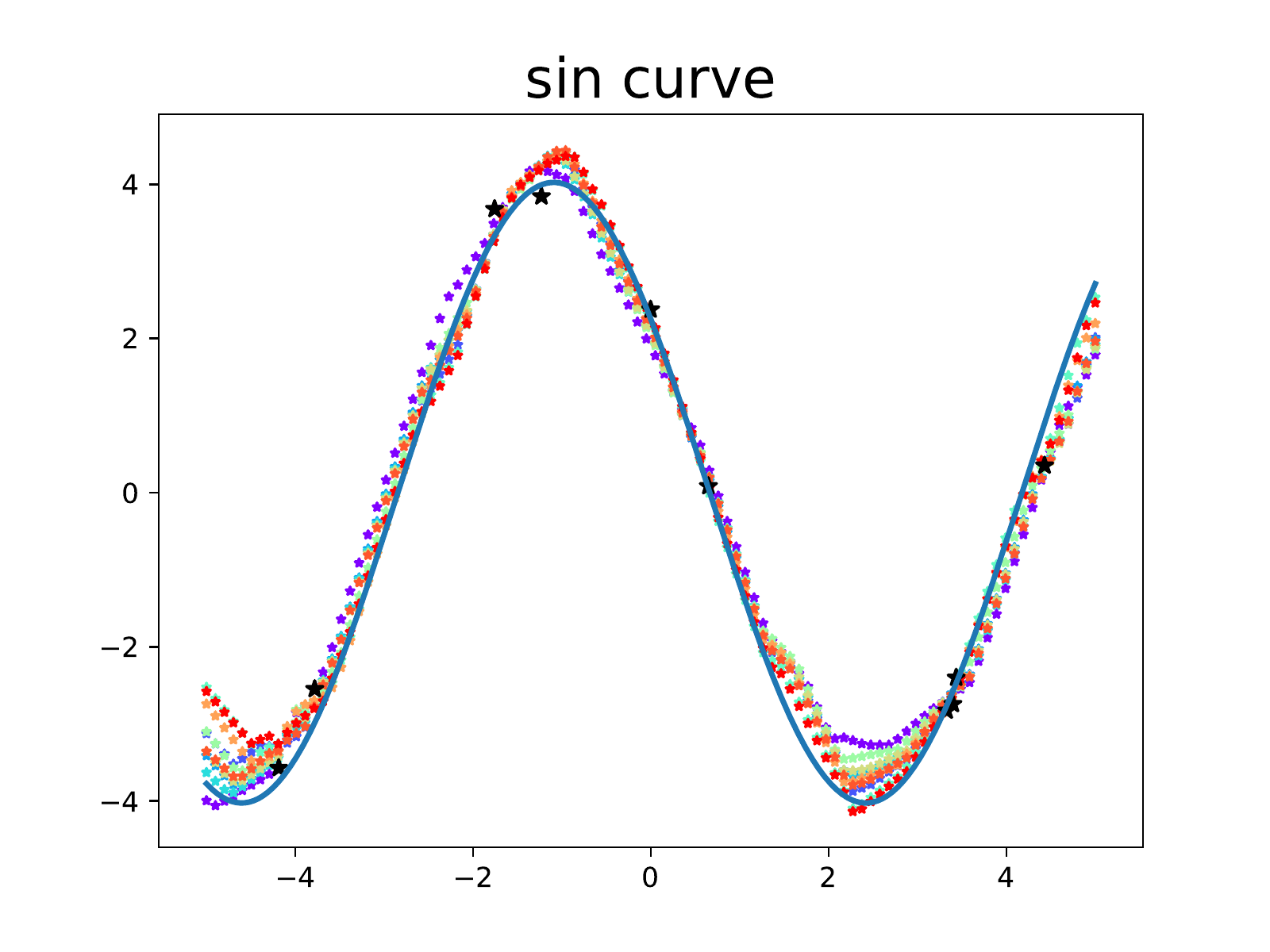}
    \end{minipage}
    \begin{minipage}[b]{0.24\textwidth} 
      \centering 
      \includegraphics[width=1\linewidth]{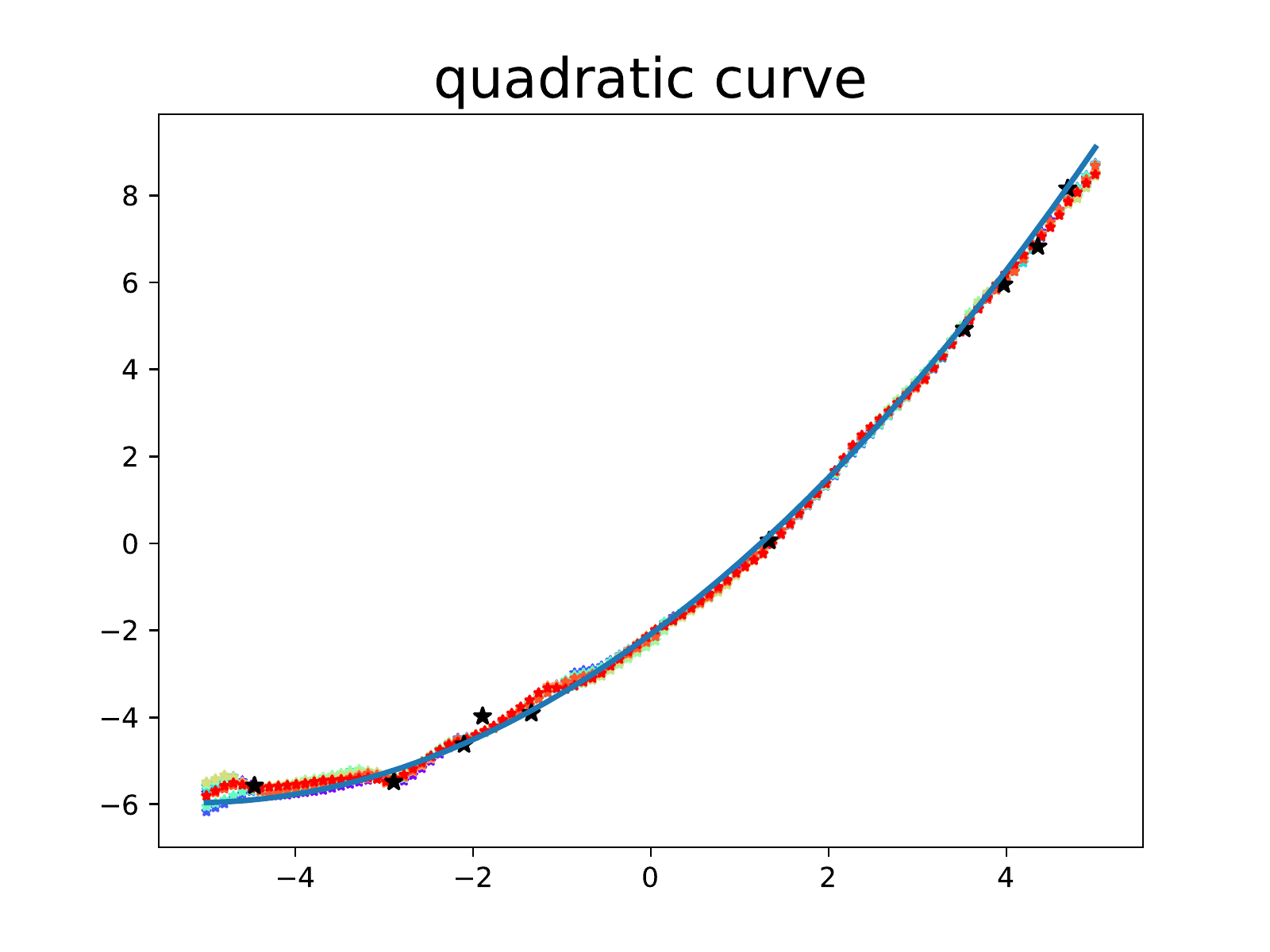}
    \end{minipage}
  \caption{\label{fig:2DReg_10shot}Qualitative Visualization of fitting curves. Black stars represent training set $\datatr$, 10 different samples of fitting curves are shown as colored dotted lines. The blue solid line is the true mapping.} 
\end{figure*}

\textbf{More results for 2D regression.}
During meta-training, we fixed the size of training set $|\datatr|$ as 10, the standard deviation for Gaussian noise $\sigma$ to be 0.3, during meta-testing, we can decrease the size of training set or increase the noise level such that tasks ambiguity can be more concerning, we visualize them in Figure~\ref{fig:2DReg_2_5shot}. The model can effectively reason over ambiguity as we vary the size of the training data or noise level. The sampled functions tend to span wider space as $|\datatr|$  decreases or the noise level increases. However, they stay faithful around those annotated training data.

\begin{figure*}[h]
  \centering 
      \includegraphics[width=1.0\linewidth]{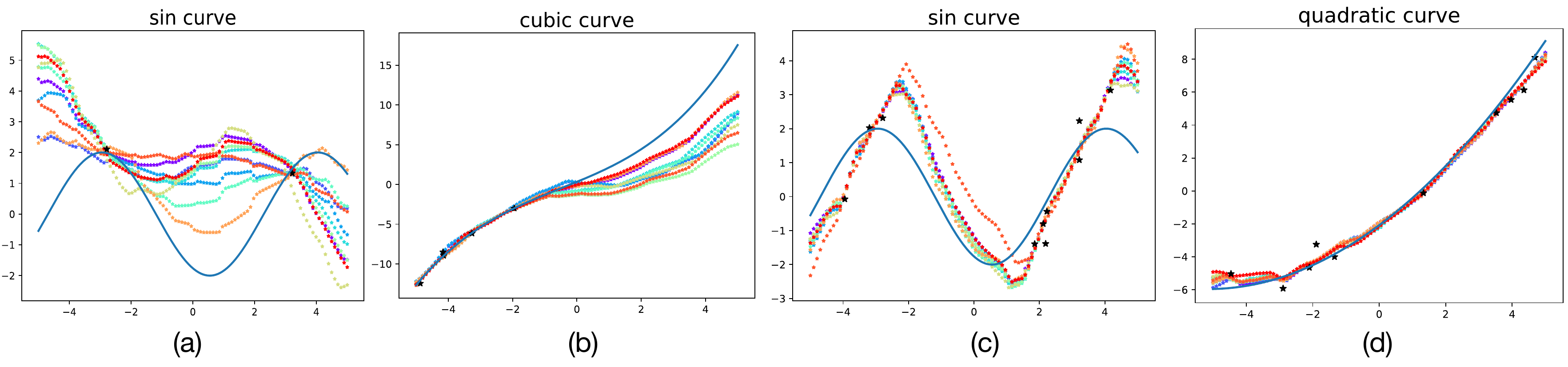}
\caption{\label{fig:2DReg_2_5shot}Few-shot 2D regression with various number of training data and noise level. (a) $|\datatr|=2, \sigma=0.3$ (b)  $|\datatr|=5, \sigma=0.3$, (c) $|\datatr|=10, \sigma=0.8$, (d)  $|\datatr|=10, \sigma=0.1$. Black star represents training data, dashed lines characterize different sampled models, the blue curve is the true mapping.}
\end{figure*}

\textbf{NOAA GSOD Dataset Details}. The data is available at \url{https://data.noaa.gov/dataset/dataset/global-surface-summary-of-the-day-gsod}. The dataset is large, so we reduce the size while preserving a wide range of years by using every $10$th year from $1969-2019$. Each file in the unzipped dataset corresponds to one year of data at a particular station. Files that do not contain at least $40$ days of data are ignored. Task number $i$ is created in the following way:
\begin{enumerate}
    \item We sample $40$ days of data that have valid temperature entires.
    \item We drop the columns ("STATION", "NAME", "TEMP\_ATTRIBUTES", "DEWP",
                "DEWP\_ATTRIBUTES",
                "PRCP\_ATTRIBUTES",
                "SLP\_ATTRIBUTES",
                "STP\_ATTRIBUTES",
                "VISIB\_ATTRIBUTES",
                "WDSP\_ATTRIBUTES",
                "MAX",
                "MIN",
                "\text{MAX}\_ATTRIBUTES",
                "MIN\_ATTRIBUTES",
                "LATITUDE", and
                "LONGITUDE")
    \item We convert the date column from (MM/DD/YYYY) to a float [0, 1] representing the time since the first day of that year.
    \item The ``FRSHTT" is a 6 bit binary string where each digit indicates the presence of fog, rain, snow, hail, thunder, and tornadoes respectively. We transform the ``FRSHTT" column into 6 binary columns.
    \item The GSOD dataset reports missing values with all $9$s, e.g. $99.99$, or $999.9$. We find and replace these values with $0.0$. We also replace NaN entries with $0.0$.
    \item The units of some input variables are adjusted to bring their values down to a smaller range. Pressure variables (``SLP" and ``STP") are converted from millibars to bars. Elevation is changed from meters to kilometers.
    \item The ``TEMP" variable is split from the data to become our target value.
\end{enumerate}

We use a 42k/5k/1k split to divide the files into train, val and test sets.

\textbf{Model architecture for weather prediction.}
Similar to 2D regression, the feature learner has two linear layers with 100 neurons followed by ReLU activation funcion. The mapping to task representation $\taskZ$ contains 3 layers with hidden dimension 40. 80, 200. The augmented dimension is set to be 20.

\textbf{Model runtime and compute.} The model trains on one GTX 2080 card. Training times vary by experiment, ranging from a few hours to a day.